\documentclass{article} 
\usepackage{iclr2025_conference,times}

\usepackage{amsmath,amsfonts,bm}

\def\eqref#1{equation~\ref{#1}}

\def\1{\bm{1}}

\def\eps{{\epsilon}}

\def\vv{{\bm{v}}}

\def\vx{{\bm{x}}}
\def\vy{{\bm{y}}}

\DeclareMathAlphabet{\mathsfit}{\encodingdefault}{\sfdefault}{m}{sl}
\SetMathAlphabet{\mathsfit}{bold}{\encodingdefault}{\sfdefault}{bx}{n}

\def\gG{{\mathcal{G}}}

\def\sD{{\mathbb{D}}}

\def\sH{{\mathbb{H}}}

\def\sS{{\mathbb{S}}}

\newcommand{\E}{\mathbb{E}}

\newcommand{\R}{\mathbb{R}}

\usepackage{hyperref}
\usepackage{url}

\usepackage{booktabs}
\usepackage{multirow}
\usepackage{adjustbox}
\usepackage{xcolor} 
\usepackage{tikz}

\title{sHGCN: Simplified hyperbolic graph convolutional neural networks}

\iclrfinalcopy

\author{Pol Arévalo \thanks{Work done during an internship at Nostrum Biodiscovery.},  Alexis Molina, Álvaro Ciudad \\
Department of Artificial Intelligence\\
Nostrum Biodiscovery S.L.\\
Barcelona, Spain\\
\texttt{alvaro.ciudad@nostrumbiodiscovery.com}
}

\usepackage{amsthm}
\newtheorem{theorem}{Theorem}[section]
\newtheorem{proposition}[theorem]{Proposition}
\newtheorem{lemma}[theorem]{Lemma}
\newtheorem{corollary}{Corollary}[theorem]
\newtheorem{definition}{Definition}
\usepackage{tablefootnote}

\begin{document}

\maketitle

\begin{abstract}
Hyperbolic geometry has emerged as a powerful tool for modeling complex, structured data, particularly where hierarchical or tree-like relationships are present. By enabling embeddings with lower distortion, hyperbolic neural networks offer promising alternatives to Euclidean-based models for capturing intricate data structures. Despite these advantages, they often face performance challenges, particularly in computational efficiency and tasks requiring high precision. In this work, we address these limitations by simplifying key operations within hyperbolic neural networks, achieving notable improvements in both runtime and performance. Our findings demonstrate that streamlined hyperbolic operations can lead to substantial gains in computational speed and predictive accuracy, making hyperbolic neural networks a more viable choice for a broader range of applications.
\end{abstract}

\section{Introduction}

Recent datasets with complex geometric structures have highlighted the limitations of traditional Euclidean spaces in capturing hierarchical or tree-like data found in domains like complex networks \citep{ComplexNetworks}, natural language processing \citep{NLPtasks2, NLPtasks1, NLPtasks3}, and protein interactions \citep{ProteinInteraction}. Hyperbolic spaces, with their exponential volume growth, offer a more suitable alternative for embedding such data with lower distortion \citep{LowDistortion}, effectively mirroring hierarchical structures.

Hyperbolic neural networks (HNNs) \citep{HNNOctavian} and hyperbolic graph convolutional networks (HGCNs) \citep{ChamiHGCN} have demonstrated superior performance over Euclidean models in learning from hierarchical data. However, challenges remain: hyperbolic models suffer from numerical stability issues \citep{StabilityRoundOffBoundary, NumericalStability}, and their computations, especially aggregations, are often slower than Euclidean counterparts.

To tackle these challenges, we present the sHGCN model, a streamlined version of HGCN that delivers state-of-the-art performance in link prediction, node classification, and graph regression tasks. Our model achieves significant computational efficiency gains, enabling faster processing and enhanced scalability compared to existing approaches. This makes sHGCN not only more effective but also better suited for real-world applications.

\section{Hyperbolic Graph Convolutional Neural Networks (HGCNs)}

Graph Convolutional Networks (GCNs) \citep{GCN} have recently demonstrated significant superiority over traditional machine learning models in graph-related tasks and applications, such as node classification, link prediction, and graph classification. Hyperbolic GCNs (HGCNs) have further achieved remarkable success in studying graph data, particularly with a tree-like structure, due to the unique properties of hyperbolic spaces. For a more detailed background on the geometric frameworks that support HGCNs, please refer to Appendix \ref{ap: geometry}.

Many authors have proposed various methods for implementing HGCNs as summarized in \citet{ReviewHGCNMethods}. In general, a unified HGCN can be formulated as:

\begin{align}
    h_i^{l, \mathcal{H}} &= \left(W^l\otimes^{c_{l-1}}x_i^{l-1,\mathcal{H}}\right)\oplus^{c_{l-1}}b^l \qquad & \text{(feature transform)} \label{eq:general-feat} \\
    y_i^{l, \mathcal{H}} &= AGG^{c_{l-1}}\left(h^{l, \mathcal{H}}\right)_i \qquad &\text{(neighborhood aggregation)} \label{eq:general-agg} \\
    x_i^{l, \mathcal{H}} &= \sigma^{\otimes^{c_{l-1}, c_l}}\left(y_i^{l, \mathcal{H}}\right) \qquad &\text{(non-linear activation)}, \label{eq:general-act}
\end{align}

where $\mathcal{H}$ denotes the hyperbolic space, commonly represented using either the Lorentz $\sH^n$ or Poincaré $\sD^n$ models. The choice of hyperbolic model impacts how the operations in Eq.~(\ref{eq:general-feat})-(\ref{eq:general-act}) are computed, as each model requires different adaptations to the underlying geometry. Consequently, these steps are implemented in various ways depending on the specific method proposed in the literature.

\section{Motivation}

The core motivation behind our proposed model was to address several key limitations found in existing methods that utilize hyperbolic geometry: (i) Hyperbolic operations tend to introduce numerical instability; (ii) The time performance of these methods is inefficient, making them impractical for large-scale applications. 

We build upon the HGCN model introduced by \citet{ChamiHGCN} as the foundation for our work. This framework provides flexibility by supporting both the Poincaré ball and Lorentz models of hyperbolic geometry. In our study, we opted for the Poincaré ball model due to the numerical instabilities encountered which are further analyzed in Appendix \ref{ap:error_analysis}.

To enhance computational efficiency, we employed fixed weights \( w_{ij} = 1 \) in the aggregation step for all \( i,j \), and adopted origin-based aggregation in hyperbolic space, resulting in the (HGCN-AGG$_0$) model. This approach addresses runtime concerns while preserving the advantages of hyperbolic operations.

To gain deeper insights and identify optimization opportunities, we revisited the original formulation and derived a consolidated matrix equation for the entire framework. The general message-passing rule of the HGCN-AGG$_0$ model is:

\begin{align} 
        h_i^{l, \sD} &= \exp_0^{c_{l-1}}\left(W^l\log_0^{c_{l-1}}(x_i^{l-1,\sD})\right)\oplus^{c_{l-1}} \exp_0^{c_{l-1}}(b^{l}) \qquad &\text{(feature transform)} \label{eq:feat-passing}\\
        y_i^{l, \sD} &= \exp_0^{c_{l-1}}\left(\sum_{j\in\mathcal{N}(i)\cup\{i\}}\log_0^{c_{l-1}}(h_j^{l, \sD})\right) \qquad &\text{(aggregation at the origin)} \label{eq:agg-passing} \\
        x_i^{l, \sD} &= \exp_0^{c_l}\left(\sigma\left(\log_0^{c_{l-1}}(y_i^{l, \sD})\right)\right) &\text{(non-linear activation)} \label{eq:act-passing}.
\end{align}

To simplify the notation, we define \(\log = \log_0^{c_{l-1}}\) and \(\exp = \exp_0^{c_{l-1}}\). Combining these definitions, we obtain the following consolidated expression from Eq.~(\ref{eq:feat-passing})-(\ref{eq:act-passing}):

\begin{equation} \label{eq:single_eq_hgcn}
    H^l = \boldsymbol{\exp_0^{c_l}}\left(\sigma\left(\boldsymbol{\log}\left(\boldsymbol{\exp}\left(\tilde{A}\log\left(\exp(W^l\boldsymbol{\log}(H^{l-1})\right)\oplus^{c_{l-1}}\exp(b^l)\right)\right)\right)\right),
\end{equation}

where \(\tilde{A} = \hat{D}^{-1} \hat{A}\), with \(\hat{A} = A + I_N\) and \(\hat{D}_{ii} = \sum_j \hat{A}_{ij}\). \(W^l\) is a layer-specific trainable weight matrix, and \(\sigma(\cdot)\) denotes the activation function. The feature matrix at the \(l^{\text{th}}\) layer is \(H^l \in \mathbb{R}^{N \times D}\). Additionally, \(H^{(0)} = \exp_0^{c_0}(X)\) represents the matrix of Euclidean node feature vectors \(X_i\) transformed into their hyperbolic counterparts.

From Eq.~\ref{eq:single_eq_hgcn}, the bold transformations can be omitted as they correspond to the identity \( \log(\exp(x)) = x \) in the Poincaré space. While these operations theoretically do not affect the model's outcomes, they introduce unnecessary computations that slow it down. Moreover, due to precision issues, this identity breaks down for values far from the origin, leading to numerical errors that degrade accuracy, as shown in Fig.~\ref{fig:log-map}. For a more detailed analysis, refer to Appendix \ref{ap:poincare_analysis}.

\begin{figure}[h!]
    \centering
    \begin{minipage}{0.45\textwidth}
        \centering
        \begin{tikzpicture}
            \draw[->, thick] (-1.5, 0) -- (1.5, 0) node[right] {$x$};
            \draw[->, thick] (0, -1.5) -- (0, 1.5) node[above] {$y$};
    
            \filldraw[red] (0,1.3) circle (2pt);
            \filldraw[blue] (0, 0.7) circle (2pt);
            \filldraw[red] (0, 1.1) circle (2pt);
    
            \draw[->, thick, red] (3.0, 0.5) -- (5.3, 0.5) node[midway, above] {$\text{exp}$};
            \draw[<-, thick, blue] (3.0, -0.5) -- (5.3, -0.5) node[midway, below] {$\text{log}$};
            \node at (-1.4, 1.6) {$\mathbb{R}^2$};
        \end{tikzpicture}
    \end{minipage}%
    \hspace{0.5cm}  
    \begin{minipage}{0.45\textwidth}
        \centering
        \begin{tikzpicture}
            \draw[thick] (0,0) circle(0.6);
            \draw[->, thick] (-1.5, 0) -- (1.5, 0) node[right] {$x$};
            \draw[->, thick] (0, -1.5) -- (0, 1.5) node[above] {$y$};
    
            \filldraw[red] (0,0.5967) circle (2pt);
            \node at (-1.4, 1.6) {$\mathbb{D}^2$};
        \end{tikzpicture}
    \end{minipage}
    \caption{Composition $\log \circ \exp$ maps becomes inaccurate as points move away from the origin.}
    \label{fig:log-map}
\end{figure}
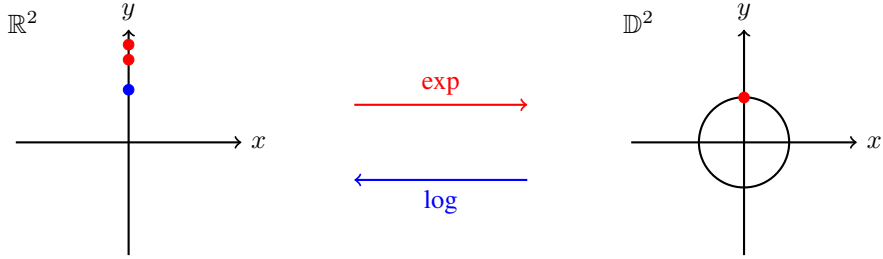

\section{Proposed method} \label{sec:ProposedModel}

\subsection{Feature transform}

A linear transformation involves multiplying the embedding vector $x$ by a weight matrix $W$ and adding a bias term $b$. For feature transformation, we use matrix-vector multiplication in the Poincaré ball model along with Möbius addition  (see Appendix \ref{ap: geometry} for a detailed definition of these operations). The bias term $b$ is defined as a Euclidean vector. Let $W$ be a \( d' \times d \) weight matrix in Euclidean space. Since both $Wx$ and $b$ are in Euclidean space, which is isomorphic to the tangent space \( \mathcal{T}_0 \sD_{c}^n \), we can use exponential maps to project these into hyperbolic space. Thus, the feature transformation is defined as follows:

\begin{equation}
    f(x) = \exp_0^c\left(Wx\right)\oplus^c\exp_0^c(b).
\end{equation}

\subsection{Neighborhood aggregation}

Aggregation is crucial in HGCNs for capturing neighborhood structures and features. Our proposed method simplifies the aggregation process by using fixed weight values \( w_{ij} = 1 \), performed directly in the tangent space at the origin, significantly enhancing computational efficiency. Traditional HGCN methods face challenges due to the computational burden of dynamic weight calculations. In fact, as demonstrated in the results section \ref{sec:results}, one notable case highlights how link prediction can result in an Out-Of-Memory error.
By streamlining these operations, our algorithm effectively addresses these issues and accelerates performance. The aggregation operation is then defined as follows:

\begin{equation}  \label{eq:hgcn_agg}
    AGG^c(x)_i = \sum_{j\in\mathcal{N}(i)\cup \{i\}}\log_0^c(h_j).
\end{equation}

\subsection{Non-linear activation}

Since the output from our neighborhood aggregation is in Euclidean space, Eq.~(\ref{eq:hgcn_agg}), we apply a standard non-linear activation function. This approach is highly advantageous, as it frees us from being restricted to activation functions that preserve hyperbolic geometry, allowing greater flexibility in choosing among a wider range of non-linear activations. The non-linear activation is then as follows:

\begin{equation} \label{eq:non-linear-act}
    \sigma^{\otimes^c}(x) = \sigma(x).
\end{equation}

\subsection{Trainable curvature}

Similar to the HGCN model, our method introduces learnable curvature at each layer, allowing for a better capture of the underlying geometry at each layer.

\subsection{Architecture}

Given a graph $\gG=\left(\mathcal{V}, \mathcal{E}\right)$ with $N$ nodes $\{x_i^\E\}_{i\in\mathcal{V}}$ and edges $(x_i^\E, x_j^\E) \in\mathcal{E}\subseteq\mathcal{V}\times\mathcal{V}$. Message passing rule of the new method sHGCN at layer $l$ for node $i$ then consists of:

\begin{align} 
    h_i^{l, \sD} &= \exp_0^{c_{l-1}}(W^lx_i^{l-1,\E})\oplus^{c_{l-1}}\exp_0^{c_{l-1}}(b^{l,\E}) \qquad &\text{(feature transform)} \label{eq:feat-transf}  \\
    y_i^{l, \E} &= \sum_{j\in\mathcal{N}(i)\cup \{i\}} \log_0^{c_l}(h_j^{l, \sD}) \qquad &\text{(neighborhood aggregation)} \label{eq:neigh-agg} \\
    x_i^{l, \E} &= \sigma\left(y_i^{l,\E}\right) \qquad &\text{(non-linear activation)}. \label{eq:activ}
\end{align}

Combining Eq.~(\ref{eq:feat-transf})-(\ref{eq:activ}) at a matrix level we obtain: 

\begin{equation}
    H^{l} = \sigma\left(\tilde{A}\log_{0}^{c_{l-1}}\left(\exp_{0}^{c_{l-1}}(W^{l}H^{l-1})\oplus^{c_{l-1}} \exp_{0}^{c_{l-1}}(b^{l-1}) \right) \right),  \ \forall{l\geq1},     
\end{equation}

where, \(\tilde{A}\), \(W^l\), and \(\sigma(\cdot)\) are defined as in the HGCN model, with the distinction that \(H^{(l)} \in \mathbb{R}^{N \times D}\). Specifically, \(H^{(0)} = X\) is the matrix of Euclidean node feature vectors \(X_i\). To better illustrate the architecture and flow of the sHGCN message-passing method, we present a diagram of this model in Fig.~\ref{fig:architecture}.

\begin{figure}[h!]
    \centering
    \includegraphics[width=\linewidth]{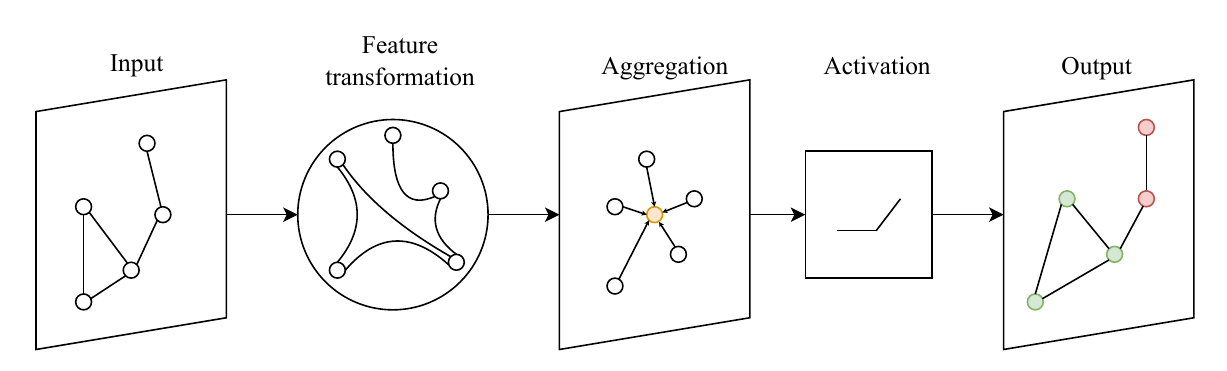}
    \caption{The proposed sHGCN architecture.}
    \label{fig:architecture}
\end{figure}

\section{Results} \label{sec:results}

For details on the experimental setup, please refer to Appendix \ref{ap:experimental_setup}. 

\textbf{Link Prediction.}  
The performance difference is particularly noticeable when comparing our new method to hyperbolic-based models on specific datasets. This is most evident in the Disease dataset (see Table \ref{tab:lp_nc_results}), where our proposed model achieves a test \textbf{AUC score of 94.6}, significantly outperforming the previous HGCN model’s score of \textbf{90.8}. While the performance of our method on other datasets is comparable to prior models, the Disease dataset stands out, highlighting the effectiveness of our approach in tasks involving hierarchical relationships.

Table \ref{tab:performance_comparison} further illustrates the computational performance of our model. The newly proposed approach shows notable improvements in efficiency, with speedups ranging from \textbf{3\% to 96\%} over the HGCN-AGG$_0$ model, and from \textbf{69\% to 146\%} compared to the HGCN-ATT$_0$ model. However, it is important to note that the HGCN-ATT$_0$ model, which computes dynamic weights, faces Out-Of-Memory (OOM) issues, as it requires significantly more memory. This makes it unsuitable for real-world application tasks, where memory limitations and efficiency are critical factors.

The superior performance of our method in link prediction tasks can be attributed to several key factors. Notably, operating in Euclidean space offers an advantage when applying Euclidean distance in the decoder. Additionally, the error accumulation caused by the composition \( \log(\exp(x)) \) in hyperbolic models exacerbates performance gaps, further highlighting the effectiveness of our approach.

\textbf{Node Classification.}  
In node classification, the results were similar across models, due to the influence of the MLP classifier at the end of the pipeline.
This classifier often dominates the overall performance, leading to similar results.

In terms of time performance, the newly proposed sHGCN model shows significant improvements, with speedups ranging from \textbf{49\% to 56\%} over the HGCN-AGG$_0$ model, and from \textbf{45\% to 127\%} compared to the HGCN-ATT$_0$ model. This gain in performance can be attributed to several factors. Firstly, the sHGCN model requires fewer operations compared to its counterparts. Secondly, one key advantage of the new model is that the embeddings are already in Euclidean space, eliminating the need for additional steps to map embeddings back to Euclidean space, as required in the other hyperbolic models. 

\begin{table*}[!ht]
\centering
\caption{ROC AUC for Link Prediction (LP), and F1 score (disease and airport) and accuracy (pubmed and cora) for Node
Classification (NC) tasks. All settings and results are reported from \citet{ChamiHGCN}.}
\label{tab:lp_nc_results}
\resizebox{\textwidth}{!}{
\begin{tabular}{lcccccccccc}
\toprule
\multirow{2}{*}{\textbf{Dataset Hyperbolicity $\delta$}} & \multicolumn{2}{c}{\textbf{DISEASE} ($\delta=0$)} & \multicolumn{2}{c}{\textbf{AIRPORT} ($\delta=1$)} & \multicolumn{2}{c}{\textbf{PUBMED} ($\delta=3.5$)} & \multicolumn{2}{c}{\textbf{CORA} ($\delta=11$)} \\
\cmidrule(lr){2-9}
\textbf{Method} & \textbf{LP} & \textbf{NC} & \textbf{LP} & \textbf{NC} & \textbf{LP} & \textbf{NC} & \textbf{LP} & \textbf{NC} \\
\midrule
\textbf{Shallow Models} \\
EUC & 59.8 $\pm$ 2.0 & 32.5 $\pm$ 1.1 & 92.0 $\pm$ 0.0 & 60.9 $\pm$ 3.4 & 83.3 $\pm$ 0.1 & 48.2 $\pm$ 0.7 & 82.5 $\pm$ 0.3 & 23.8 $\pm$ 0.7 \\
HYP \citep{PoincareEmbeddings} & 63.5 $\pm$ 0.6 & 45.5 $\pm$ 3.3 & 94.5 $\pm$ 0.0 & 70.2 $\pm$ 0.1 & 87.5 $\pm$ 0.1 & 68.5 $\pm$ 0.3 & 87.6 $\pm$ 0.2 & 22.0 $\pm$ 1.5 \\
EUC-MIXED & 49.6 $\pm$ 1.1 & 35.2 $\pm$ 3.4 & 91.5 $\pm$ 0.1 & 68.3 $\pm$ 2.3 & 86.0 $\pm$ 1.3 & 63.0 $\pm$ 0.3 & 84.4 $\pm$ 0.2 & 46.1 $\pm$ 0.4 \\
HYP-MIXED & 55.1 $\pm$ 1.3 & 56.9 $\pm$ 1.5 & 93.3 $\pm$ 0.0 & 69.6 $\pm$ 0.1 & 83.8 $\pm$ 0.3 & 73.9 $\pm$ 0.3 & 85.6 $\pm$ 0.5 & 45.9 $\pm$ 0.3 \\
\midrule
\textbf{Neural Networks (NN)} \\
MLP & 72.6 $\pm$ 0.6 & 28.8 $\pm$ 2.5 & 89.8 $\pm$ 0.5 & 68.6 $\pm$ 0.6 & 84.1 $\pm$ 0.9 & 51.5 $\pm$ 1.0 & 83.1 $\pm$ 0.5 & 51.5 $\pm$ 1.0 \\
HNN \citep{HNNOctavian} & 75.1 $\pm$ 0.3 & 41.0 $\pm$ 1.8 & 90.8 $\pm$ 0.2 & 80.5 $\pm$ 0.5 & 94.9 $\pm$ 0.1 & 54.6 $\pm$ 0.4 & 89.0 $\pm$ 0.1 & 54.6 $\pm$ 0.4 \\
\midrule
\textbf{Graph Neural Networks (GNN)} \\
GCN \citep{GCN} & 64.7 $\pm$ 0.5 & 69.7 $\pm$ 0.4 & 91.1 $\pm$ 0.5 & 81.4 $\pm$ 0.6 & 90.4 $\pm$ 0.2 & 81.3 $\pm$ 0.3 & 90.4 $\pm$ 0.2 & 81.3 $\pm$ 0.3 \\
GAT \citep{GAT} & 69.8 $\pm$ 0.3 & 70.4 $\pm$ 0.4 & 91.2 $\pm$ 0.3 & 81.5 $\pm$ 0.3 & 93.7 $\pm$ 0.4 & \textbf{83.0 $\pm$ 0.7} & \textbf{93.7 $\pm$ 0.1} & \textbf{83.0 $\pm$ 0.7} \\
SAGE \citep{SAGE} & 65.9 $\pm$ 0.3 & 69.1 $\pm$ 0.6 & 90.5 $\pm$ 0.5 & 82.1 $\pm$ 0.5 & 85.5 $\pm$ 0.6 & 77.9 $\pm$ 2.4 & 85.5 $\pm$ 0.6 & 77.9 $\pm$ 2.4 \\
SGCN \citep{SGCN} & 65.1 $\pm$ 0.2 & 69.5 $\pm$ 0.2 & 89.8 $\pm$ 0.1 & 81.0 $\pm$ 0.1 & 91.5 $\pm$ 0.1 & 81.0 $\pm$ 0.1 & 91.5 $\pm$ 0.1 & 81.0 $\pm$ 0.1 \\
\midrule
\textbf{Hyperbolic GCN (HGCN)} \\
HGCN \citep{ChamiHGCN} & 90.8 $\pm$ 0.3 & 74.5 $\pm$ 0.9 & \textbf{96.4 $\pm$ 0.1} & \textbf{90.6 $\pm$ 0.2} & \textbf{96.3 $\pm$ 0.0} & 80.3 $\pm$ 0.3 & 92.9 $\pm$ 0.1 & 79.9 $\pm$ 0.2 \\
HGCN-AGG$_0$ & 81.6 $\pm$ 7.5 & 86.5 $\pm$ 6.0 & 93.6 $\pm$ 0.4 & 85.8 $\pm$ 1.5 & 95.1 $\pm$ 0.1 & 	80.2 $\pm$ 0.9 & 93.1 $\pm$ 0.3 & 76.6 $\pm$ 1.2 \\
HGCN-ATT$_0$ & 82.0 $\pm$ 10.0 & \textbf{88.0 $\pm$ 1.2} & 93.9 $\pm$ 0.6 & 85.6 $\pm$ 1.6 & OOM & OOM & 93.1 $\pm$ 0.3 & 76.6 $\pm$ 1.8 \\
\textbf{Ours: sHGCN} & \textbf{94.6 $\pm$ 0.6} & 86.6 $\pm$ 5.8 & 94.7 $\pm$ 0.4 & 85.5 $\pm$ 2.1 & 95.7 $\pm$ 0.3 & 80.5 $\pm$ 1.1 & 93.4 $\pm$ 0.4 & 76.5 $\pm$ 1.1 \\ 
\midrule
\end{tabular}
}
\footnotesize{An Out-of-Memory (OOM) error occurred in the GPU memory (see \ref{ap:experimental_setup} for the hardware specifications used). \\ It is important to note that we were unable to reproduce the original results of the HGCN model, so we present the results reported in the paper.}
\end{table*}

\begin{table*}[!ht]
\centering
\caption{Comparisons of speed-up between sHGCN and other existing models.}
\label{tab:performance_comparison}
\resizebox{\textwidth}{!}{
\begin{tabular}{lccccccccccc}
\toprule
\multirow{2}{*}{\textbf{Dataset Hyperbolicity $\delta$}} & \multicolumn{2}{c}{\textbf{DISEASE} ($\delta=0$)} & \multicolumn{2}{c}{\textbf{AIRPORT} ($\delta=1$)} & \multicolumn{1}{c}{\textbf{PUBMED} ($\delta=3.5$)} & \multicolumn{2}{c}{\textbf{CORA} ($\delta=11$)} \\
\cmidrule(lr){2-8}
\textbf{Method} & \textbf{LP} & \textbf{NC} & \textbf{LP} & \textbf{NC} & \textbf{LP} & \textbf{LP} & \textbf{NC} \\
\midrule
HGCN-AGG$_0$ & 0.0404 $\pm$ 0.0008 & 0.0366 $\pm$ 0.0002 & 0.0753 $\pm$ 0.0010 & 0.0384 $\pm$ 0.0006 & 0.1149 $\pm$ 0.0049 & 0.0514 $\pm$ 0.0009 & 0.0382 $\pm$ 0.0007 \\
HGCN-ATT$_0$ & 0.0547 $\pm$ 0.0009 & 0.0356 $\pm$ 0.0002 & 0.0946 $\pm$ 0.0012 & 0.0559 $\pm$ 0.0002 & OOM & 0.0583 $\pm$ 0.0008 & 0.0499 $\pm$ 0.0004 \\
\textbf{Ours: sHGCN} & \textbf{0.0273 $\pm$ 0.0004} & \textbf{0.0245 $\pm$ 0.0005} & \textbf{0.0384 $\pm$ 0.0006} & \textbf{0.0246 $\pm$ 0.0002} & \textbf{0.1116 $\pm$ 0.0047} &  \textbf{0.0345 $\pm$ 0.0004} & \textbf{0.0258 $\pm$ 0.0004}
 \\ 

\midrule

Speedup (HGCN-AGG$_0$ vs sHGCN) & 1.48 $\pm$ 0.04 & 1.49 $\pm$ 0.03 & 1.96 $\pm$ 0.06 & 1.56 $\pm$ 0.03 & 1.03 $\pm$ 0.07 & 1.49 $\pm$ 0.04 & 1.48 $\pm$ 0.04 \\
Speedup (HGCN-ATT$_0$ vs sHGCN) & 2.00 $\pm$ 0.04 & 1.45 $\pm$ 0.03 & 2.46 $\pm$ 0.05 & 2.27 $\pm$ 0.02 & OOM & 1.69 $\pm$ 0.03 & 1.93 $\pm$ 0.03 \\

\midrule

\end{tabular}
}
\end{table*}

\textbf{Graph Regression.} Additionally, we have conducted experiments with Graph Regression due to the possibility of representing molecules as graphs. We believe that these molecules may exhibit high hyperbolicity, and that our sHGCN model could be particularly useful in such cases, as demonstrated by the results in the Appendix \ref{ap:GraphRegression}.

\section{Conclusions}

In this paper, we introduce sHGCN, a streamlined architecture built upon Hyperbolic Graph Convolutional Networks (HGCNs). Our empirical results demonstrate that sHGCN delivers performance on par with, and in certain instances superior to, traditional HGCNs, particularly in the context of hyperbolic graphs, such as those employed in disease modeling. Additionally, sHGCN exhibits a significant improvement in computational efficiency, making it a highly practical and scalable solution for real-world applications involving graph-structured data.

\bibliography{iclr2025_conference}
\bibliographystyle{iclr2025_conference}

\newpage

\section*{Reproducibility Statement}

The source code for reproducing the work presented here is available at \url{https://github.com/pol-arevalo-soler/Hyperbolic-GCN}. This repository focuses on the implementation of our method for Node Classification and Link Prediction tasks. For experiments related to Graph Regression, we have developed a separate repository, which can be found at \url{https://github.com/pol-arevalo-soler/Benchmarking_gnn}.

Both repositories are implemented using PyTorch \citet{Pytorch}, and include the necessary infrastructure to conduct experiments and generate results. To ensure reproducibility, we have systematically logged every hyperparameter along with the random seeds used in our experiments.

\appendix

\section{Preliminaries} \label{ap: geometry}

\subsection{Problem setting}

Without losing generality, we focus on explaining graph representation learning using a single graph as our example. Given a graph $G = (\mathcal{V}, \mathcal{E})$ with vertex set $\mathcal{V}$, edge set $\mathcal{E}$, $\mathcal{N}(i) = \{j: (i, j) \in \mathcal{E}\}$ denote a set of neighbors of $i\in\mathcal{V}$ and $(x_i)_{i\in\mathcal{V}}$ $d$-dimensional input features, the goal is to learn a mapping $f$ which maps nodes to embedding vectors:

\begin{equation}
    f:\left(\mathcal{V}, \mathcal{E}, (x_i)_{i\in\mathcal{V}}\right) \longrightarrow Z \in \R^{|\mathcal{V}|\times d'},
\end{equation}

where $d' \ll |\mathcal{V}|$.

\subsection{Geometry space}

Geometric spaces are defined by their curvature \(c\): negative for hyperbolic space \(\sH\), positive for hyperspherical space \(\sS\), and zero for Euclidean space \(\E\). Among the five isometric models of hyperbolic space, we chose the Poincaré Ball model, commonly used in research \citep{PoincareEmbeddings, HNNOctavian}. This model is favored for its well-defined mathematical operations and closed-form expressions for key geometric properties like distance and angle, enabling efficient computations while preserving hyperbolic geometry.

\subsection{Poincaré Ball} \label{ap:PoincareBall}

The \textit{Poincaré ball model} \((\sD^n, g)\) is defined as the manifold $\sD^n := \{x \in \mathbb{R}^n : ||\vx|| < 1\}$ equipped with the following Riemannian metric:

\begin{equation}
    g_{x}^{\sD} = \lambda_{x}^{2}g^\E, \quad \text{where} \ \lambda_x := \frac{2}{1 - ||\vx||^2} \ \text{and} \ g^\E := I_N \ \text{euclidean metric tensor}.
\end{equation}

Hyperbolic space can be further formalized as an approximate vectorial structure within the framework of gyrovector spaces, as established by \citet{Ungar2008MbiusGS}. This framework provides a sophisticated non-associative algebraic formalism for hyperbolic geometry, analogous to how vector spaces form the algebraic basis for Euclidean geometry. The application of gyrovector spaces allows for the analysis and manipulation of hyperbolic geometries through algebraic operations, enhancing our understanding of their geometric properties.

To generalize the Poincaré ball model, we define\footnote{This notation follows the usage in the paper by \citet{HNNOctavian}} \(\sD_{c}^n\) for \(c \geq 0\) as $\sD_{c}^n := \{x \in \mathbb{R}^n : c ||\vx||^2 < 1\}.$

This generalization extends the applicability of the Poincaré ball model, allowing it to encompass a wider range of hyperbolic structures and facilitating the exploration of geometric phenomena under different curvature conditions.

For two points \( x, y \in \sD_{c}^{n} \), the \textbf{Möbius addition} is defined to perform vector addition in the hyperbolic space \(\sD_{c}^{n}\):

\begin{equation} \label{eq:mobius_add}
    x \oplus^c y := \frac{\left(1 + 2c\langle x, y\rangle + c||\vy||^2\right)x + \left(1 - c||\vx||^2\right)y}{1 + 2c\langle x, y\rangle + c^2||\vx||^2||\vy||^2}.
\end{equation}

This addition operation \ref{eq:mobius_add} enables the definition of the \textbf{distance} between the two points, which is measured along the geodesic connecting them:

\begin{equation}
    d_\sD^c(x, y) := \left( \frac{2}{\sqrt{c}} \right) \tanh^{-1} \left( \sqrt{c} \, || - \vx \oplus^c \vy || \right).
\end{equation}

The \textbf{tangent space} is pivotal for performing essential operations, such as scalar multiplication, providing a familiar framework where these operations can be understood intuitively. For a point \( x \in \sD_{c}^{d} \), the tangent space at \( x \) is Euclidean and isomorphic to \( \mathbb{R}^{d} \). 

The transformation between the hyperbolic space and its tangent space is facilitated by the \textbf{exponential} and \textbf{logarithmic} maps. The logarithmic map, denoted as \( \log_{x}^{c} \), takes a point from the hyperbolic space and maps it into the tangent space. Conversely, the exponential map, denoted as \( \exp_{x}^{c} \), takes a vector from the tangent space and maps it back to the hyperbolic space.

These maps become particularly useful when evaluated at the origin. Specifically, for a vector \( v \in \mathcal{T}_{0} \sD_{c}^{d} \setminus \{0\} \) and a point \( y \in \sD_{c}^{d} \setminus \{0\} \), the following representations hold:

\begin{equation} \label{eq:exp_log_maps}
    \exp_{0}^c(v) := \tanh\left(\sqrt{c}||\vv||\right) \frac{v}{\sqrt{c}||\vv||}, \quad \log_{0}^{c}(y) := \tanh^{-1}\left(\sqrt{c}||\vy||\right) \frac{y}{\sqrt{c}||\vy||}.
\end{equation}

Using these exponential and logarithmic maps \ref{eq:exp_log_maps}, we can elegantly define \textbf{scalar multiplication} in hyperbolic space:

\begin{equation}
    r \otimes^{c} x := \exp_{0}^{c}\left(r \log_{0}^{c}(x)\right).
\end{equation}

Additionally, if \( M : \mathbb{R}^n \to \mathbb{R}^m \) is a linear map (which we can represent as a matrix), we can define \textbf{matrix-vector multiplication} in the context of hyperbolic space for any \( x \in \sD_{c}^n \), provided \( Mx \neq 0 \):

\begin{equation} \label{eq:poincare_mat_mul}
    M^{\otimes^c}(x) := \exp_0^c\left(M\log_0^c(x)\right).
\end{equation}

\textbf{Projections.} We recall projection to the Poincaré Ball model and its corresponding tangent space. A point $x\in\R^n$ can be projected on the Poincaré Ball model $\sD^n_c$ with:

\begin{equation} \label{eq:projection}
    \Pi_{\R^d\rightarrow\sD^n_c}(x):=\frac{x}{||\vx||}\delta,
\end{equation}

where $\delta$ is defined as $\frac{1-\eps}{\sqrt{c}}$. Here $\eps$ denotes a small number $10^{-k}$, for a chosen $k>0$. 

A point $v\in\R^d$ can be projected to $\mathcal{T}_0\sD^d_c$ via the \textbf{identity} $id(x)=x$ as the tangent space is isomorphic to $\R^d$. 

\subsection{Lorentz Model} \label{ap:LorentzModel}

We begin with a brief recap of the Lorentz model, providing the foundational concepts necessary for our discussion in Section \ref{ap:lorentz_error}. For a more detailed explanation, we refer the reader to \citet{ChamiHGCN}.

First, we introduce the \textbf{Minkowski inner product}. Let $\langle \cdot, \cdot\rangle_\mathcal{L}:\R^{d+1}\times\R^{d+1}\rightarrow \R$, denote the Minkowski inner product, then we can define:

\begin{equation}
    \langle x, y \rangle := -x_0y_0+x_1y_1+\dots +x_ny_n.
\end{equation}

We denote $\mathbb{H}^{d, K}$ the Lorentz model in $d$ dimensions with constant negative curvature $-1/K \ (K>0)$, and $\mathcal{T}_x\mathbb{H}^{d,K}$ the Euclidean tangent space centered at point x as:

\begin{equation}
    \mathbb{H}^{d,K}:=\{x\in\R^{d+1}:\langle x, x\rangle_\mathcal{L}=-K, \ x_0 > 0\} \hspace{4mm} \mathcal{T}_x\mathbb{H}^{d,K}:=\{v\in R^{d+1}:\langle v, x \rangle_\mathcal{L} = 0\}. 
\end{equation}

As we have seen in the Poincaré Ball model, Section \ref{ap:PoincareBall}, we can define the \textbf{exponential map}, which sends a point from the tangent space to the Lorentz model. Let $\bar{o}:=\{\sqrt{K}, 0, \dots, 0\} \in \mathbb{H}^{d, K}$ denote the north pole (origin) in $\mathbb{H}^{d, K}$. Since, $\langle (0, x^\mathbb{E}), \bar{o} \rangle_\mathcal{L} = 0$, the point $(0, x^\mathbb{E})$ lies in the tangent space $\mathcal{T}_{\bar{o}}\mathbb{H}^{d, K}$. Therefore, we can define the exponential map at the origin as the following equation as follows: 

\begin{equation}
    \exp_0^K\left((0, x^\mathbb{E})\right) = \left( \sqrt{K}\cosh\left( \frac{\|\vx^\mathbb{E}\|_2}{\sqrt{K}}\right), \sqrt{K}\sinh\left(\frac{\|\vx^\mathbb{E}\|_2}{\sqrt{K}}\right)\frac{x^\mathbb{E}}{\|\vx^\mathbb{E}\|_2}\right).
\end{equation}

\section{Error analysis} \label{ap:error_analysis}

\subsection{Error analysis in Poincaré Ball model} \label{ap:poincare_analysis}

\begin{definition}
    Machine epsilon (often denoted as \(\epsilon\)) is a measure of the smallest positive number that, when added to 1, yields a result different from 1 in floating-point arithmetic. It quantifies the precision limit of a floating-point representation in a computer system.

    Mathematically, machine epsilon can be defined as:

    \[
    \epsilon = \text{min} \{ x > 0 : 1 + x \neq 1 \}.
    \]
\end{definition}   

In PyTorch \citep{Pytorch}, we can obtain the value of epsilon using the following line of code:

\[
\texttt{torch.finfo(torch.float16).eps}
\]

This will yield approximately \(0.0009765625 \approx 9.7\times10^{-4}\) for the \texttt{float16} data type. Without loss of generality, we proceed by demonstrating the case where $x$ leads to a loss of precision. For this purpose, we assume the curvature of the Poincaré model to be $c = 1$. This assumption simplifies the analysis while preserving the generality of the results.

\begin{proposition} \label{prop:distance}
    For any point $x\in\sD^n$, if $||\vx||=1-10^{-k}$ for some positive number k, then in fact: \[
    d_{\sD^n}(0, x) = \ln(10)k + \ln(2) + O(10^{-k}).
    \]
\end{proposition}

\begin{proof}
    The proof can be found in the work of \citet{NumericalStability}.
\end{proof}

Here, $k$ is the number of bits required to avoid $x$ being rounded to the boundary. \citet{NumericalStability} provided the maximum number of bits for \texttt{float64}. However, recently, \texttt{float16} has gained significant popularity due to its reduced computation time. 

\begin{lemma} \label{lemma:max_k}
    Under float16 arithmetic system the maximum k is 4.
\end{lemma}

\begin{proof}
    As machine epsilon in PyTorch is approximately $0.0009765625$, we have that $1-10^{-k}$ will be rounded to $1$ when $10^{-k}\leq10^{-4}<9.7\times10^{-4}$. Hence, the maximum $k$ is around $4$.
\end{proof}

\begin{corollary}
    Using float16 arithmetic, we can only represent points correctly within a ball of radius $r_0\approx10$ in the Poincaré Ball.
\end{corollary}

\begin{proof}
    Directly follows from Proposition \ref{prop:distance} and Lemma \ref{lemma:max_k}. As $k=4$, we have: 

    \begin{equation*}
        d_{\sD^n}(0, x) = 4\ln(10)+\ln(2)+O(10^{-3}) \approx 10.
    \end{equation*}
\end{proof}

\begin{proposition} \label{prop:round_off}
    When using the float16 data type, any vector \( x \in \mathbb{R}^n \) with a norm \( ||\vx|| > 5 \) will be normalized to the boundary of the Poincaré Ball during the application of the exponential map \( \exp_0 \).
\end{proposition}

\begin{proof}
    We aim to determine for which \( x \in \mathbb{R}^n \) the following inequality holds:
    \begin{equation} \label{eq:initial_problem}
        d_\mathbb{D}(0, \exp_0(x)) = \cosh^{-1}\left(1 + \frac{2\|\exp_0(x)\|^2}{1 - \| \exp_0(x)\|^2}\right) > 10.
    \end{equation}

    To simplify the analysis, we can let \( y = \|\exp_0(x)\|^2 \), transforming the problem into studying the condition for which \( y \in \mathbb{R}^+ \) satisfies:
    
    \begin{equation} \label{eq:hyp_cosine}
        \cosh^{-1}\left(1 + \frac{2y}{1-y}\right) > 10.
    \end{equation}

    Since \(\cosh^{-1}(z)\) is defined only for \( z > 1 \), we require:
    \[
    1 + \frac{2y}{1 - y} > 1.
    \]

    This inequality holds true for \( y \in (0, 1) \), thus we restrict our consideration to \( y \in (0, 1) \).

    Next, applying the hyperbolic cosine function to both sides in Eq.~\ref{eq:hyp_cosine} gives us:

    \begin{equation*}
        \begin{aligned}
             1 + \frac{2y}{1 - y} > \cosh(10) &\Longleftrightarrow 2y > (\cosh(10) - 1)(1 - y) \\ &\Longleftrightarrow y (1 + \cosh(10)) > \cosh(10) - 1 \\ & \Longleftrightarrow y > \frac{\cosh(10) - 1}{\cosh(10) + 1}
        \end{aligned}
    \end{equation*}
    
    Now, reverting to the original Eq.~(\ref{eq:initial_problem}), the task has now transformed into identifying those \( x \in \mathbb{R}^n \) that satisfy the inequality:
    
    \begin{equation*} 
    \|\exp_0(x)\|^2 > \frac{\cosh(10) - 1}{\cosh(10) + 1}.
    \end{equation*}

    We know that \( \exp_0(x) = \frac{\tanh(||\vx||) x}{||\vx||} \). Thus, we rewrite the inequality as:
    \[
    \left\|\frac{\tanh(||\vx||) x}{||\vx||}\right\|^2 > \frac{\cosh(10) - 1}{\cosh(10) + 1}.
    \]

    This leads us to study the following condition for \( ||\vx|| \):
    \[
    \tanh(||\vx||)^2 > \frac{\cosh(10) - 1}{\cosh(10) + 1}.
    \]
    
    From this, we can derive:
    \[
    \tanh(||\vx||) > \sqrt{\frac{\cosh(10) - 1}{\cosh(10) + 1}}.
    \]

    Applying the inverse hyperbolic tangent function, we obtain:
    \[
    ||\vx|| > \tanh^{-1}\left(\sqrt{\frac{\cosh(10) - 1}{\cosh(10) + 1}}\right)\approx5.
    \]

    Thus, we conclude that for any \( x \in \mathbb{R}^n \) satisfying \( ||\vx|| > 5 \), the point will be mapped to the boundary of the unit ball \( \mathbb{D}^n \) after applying the exponential map \( \exp_0 \), when using \texttt{float16} arithmetic.
\end{proof}

\begin{corollary}
    Using float16 data type, the composition $\left(\log_0 \circ \exp_0\right)(x)$ is not well-defined for those $x\in\R^n$ such that $||\vx||>5$ in the unit ball $\sD^n$.
\end{corollary}

\begin{proof}
    Since for any \( x \in \mathbb{R}^n \) with \( ||\vx|| > 5 \) we have \( \exp_0(x) \in \partial(\mathbb{D}^n) \), and because \( \mathbb{D}^n \) is an open set, it follows that \( \exp_0(x) \) is on the boundary of \( \mathbb{D}^n \) and not within its interior. As the logarithmic map \( \log_0(x) \) is only defined for points \( x \in \mathbb{D}^n \) (i.e., within the open set), \( \log_0(x) \) is not well-defined for those \( x \) with \( ||\vx|| > 5 \).
\end{proof}

To mitigate these issues, \citet{ChamiHGCN} propose applying a projection, Eq.~\ref{eq:projection}, each time the 
$\exp$ and $\log$ maps are used, thereby constraining embeddings and tangent vectors to remain within the manifold and tangent spaces. However, this challenge becomes especially 
pronounced in the algorithm of the HGCN model when computing the $\exp$ map in expressions such as:

\begin{equation}\label{eq:bound_problems}
    y = A \log_0\left(\exp_0\left(W \log_0(H)\right) \oplus \exp_0(b)\right).
\end{equation}

Even with projection, for values of $y$ in Eq.~\ref{eq:bound_problems} such that 
$||\vy|| > 5$, we face the issue of distinct points collapsing to the same 
location within the Poincaré ball. This violates the fundamental relationship 
$\log_0(\exp_0(y)) = y$, as multiple points are effectively mapped to a single point on 
the boundary. Such degeneration in the representation significantly degrades model 
performance, as it undermines the ability to accurately differentiate between points in 
the embedding space.

\subsection{Error analysis in Lorentz model} \label{ap:lorentz_error}

In Section \ref{ap:poincare_analysis}, we discussed the errors associated with the Poincaré Ball model. However, the Lorentz model also faces challenges related to numerical stability and round-off errors. Notably, the Lorentz model has a reduced capacity to accurately represent points, as it can only do so within a ball of radius $r_0/2$, as emphasized by \citet{NumericalStability}. However, as we will show in  \ref{ap:threshold} and in the following proposition \ref{prop:distance_lorentz} the Lorentz model and the Poincaré Ball model share the same threshold when ensuring that the $\log_0(\exp_0(x))$ composition is well-defined.Without loss of generality, we assume the curvature to be constant and equal to $1$ in order to prove the following results.

\begin{proposition} \label{prop:distance_lorentz}
    When using the float16 data type, any vector \( x \in \mathbb{R}^n \) with a norm \( ||\vx|| > 5 \) will be normalized to the boundary of the Lorentz model during the application of the exponential map \( \exp_0 \).
\end{proposition}

\begin{proof}
    We aim to determine for which $x\in\R^n$ the following inequality holds:

    \begin{equation} \label{eq:lorentz_ineq}
        d_\mathcal{L}(\bar{o}, \exp_0(0, x^\mathbb{E})) = \cosh^{-1}\left( -\langle \bar{o}, \exp_0(0, x^\mathbb{E}) \rangle_\mathcal{L} \right) > 5.
    \end{equation}
    
    Assuming constant curvature equal to 1, we know from Section \ref{ap:LorentzModel} that \( \bar{o} = (1, 0, \dots, 0) \), and the exponential map at the origin is given by:
    
    \begin{equation*}
        \exp_0((0, x^\mathbb{E})) = \left( \cosh(\| x^\mathbb{E} \|_2), \, \sinh(\| x^\mathbb{E} \|_2) \frac{x^\mathbb{E}}{\| x^\mathbb{E} \|_2} \right).
    \end{equation*}
    
    Therefore, we have:
    
    \[
    -\langle \bar{o}, \exp_0(0, x^\mathbb{E}) \rangle_\mathcal{L} = \cosh(\| x^\mathbb{E} \|_2).
    \]
    
    Substituting this into Eq.~(\ref{eq:lorentz_ineq}), we get:
    
    \begin{equation*}
        d_\mathcal{L}(\bar{o}, \exp_0(0, x^\mathbb{E})) = \cosh^{-1}(\cosh(\| x^\mathbb{E} \|_2)) = \| x^\mathbb{E} \|_2 > 5.
    \end{equation*}

\end{proof}

In addition to these numerical issues, the Lorentz model has a conceptual flaw: the addition of bias is not well-defined. This limitation leads to errors when implementing HGCNs that incorporate bias addition. It is important to note that if we remove the bias addition from our proposed model, it reduces to a standard Graph Convolutional Network (GCN), highlighting the critical role of bias addition in our framework.

\subsection{Threshold analysis} \label{ap:threshold}

Finally, we provide Table \ref{tab:threshold_comparison}, which lists the threshold values \( t \) for different floating-point precisions. These thresholds indicate the point beyond which the composition \( \log_0(\exp_0(x)) \) becomes undefined in hyperbolic space, both Poincaré Ball and Lorentz model, when \( ||\vx|| > t \). Specifically, we consider three precision levels—\texttt{float16}, \texttt{float32}, and \texttt{float64}—to show how numerical limitations vary across data types and affect the range of well-defined operations.

\begin{table}[ht]
\centering
\caption{Threshold \( t \) for different floating-point precisions beyond which \( \log_0(\exp_0(x)) \) is not well-defined in hyperbolic space when \( ||\vx|| > t \).}
\label{tab:threshold_comparison}
\begin{tabular}{ll}
\toprule
\textbf{Floating-Point Precision} & \textbf{Threshold \( t \)} \\
\midrule
float16 & 5 \\
float32 & 9 \\
float64 & 19 \\
\bottomrule
\end{tabular}
\end{table}

\section{Experimental setup} \label{ap:experimental_setup}

\textbf{Hardware.} We performed our experiments on a machine equipped with a NVIDIA A30 GPU with 24GB of memory. This GPU is capable of running CUDA 11.6. The system was configured to utilize the GPUs for accelerating deep learning tasks and graph-based computations.

\textbf{Datasets.} We utilize a diverse set of open transductive and inductive datasets, Table \ref{tab:graph-datasets}. Our analysis focuses on Gromov’s $\delta$-hyperbolicity \citep{Gromov1, Gromov2}, a concept from group theory that quantifies the tree-like structure of a graph. A lower value of $\delta$ indicates a more hyperbolic graph, with $\delta = 0$ corresponding to trees.

\begin{table}[h!]
    \centering
    \caption{Summary of Graph Datasets used.}
    \label{tab:graph-datasets}
    \begin{tabular}{lccccc}
        \toprule
        \textbf{Name} & \textbf{Nodes} & \textbf{Edges} & \textbf{Classes} & \textbf{Node Features} & \textbf{$\delta$} \\
        \midrule
        DISEASE & 1,044 & 1,043 & 2 & 1,000 & 0 \\
        AIRPORT & 3,188 & 18,631 & 4 & 4 & 1 \\ 
        PUBMED & 19,717 & 88,651 & 3 & 500 & 3.5 \\
        CORA & 2,708 & 5,429 & 7 & 1,433 & 11 \\
        \bottomrule
    \end{tabular}
\end{table}

\textbf{Training and evaluation metric.} All experiments were conducted following the settings and evaluation metrics described in the HGCN model. Specifically for the settings, we employed a 2 layer architecture, set the embedding dimension to 16, used ReLU activation, and established a learning rate of 0.01. Model is optimized using Adam \citep{Adam}. More information about the entire hyperparameter setup can be found in our GitHub repository. \footnote{Code available at \url{https://github.com/pol-arevalo-soler/Hyperbolic-GCN}. We provide implementations for all baselines, along with instructions on how to reproduce the results presented in the paper.}

\textbf{Models.} We have compared the newly proposed sHGCN model with the HGCN model, the HGCN-AGG$_0$ model, as well as the HGCN-ATT$_0$ model. The latter corresponds to applying HGCN with origin-based aggregation, where the weights \(w_{ij}\) for all \(i, j\) are computed using the same method employed in the HGCN model. The results for HGCN-AGG$_0$ and HGCN-ATT$_0$ were obtained using the code provided by \citep{ChamiHGCN} \footnote{Unfortunately, we were unable to reproduce the HGCN model, therefore we report the values as reported in the original HGCN model}. 

\textbf{Link prediction.} As in standard HGCN models for link prediction we used the Fermi-Dirac decoder \citep{ComplexNetworks, PoincareEmbeddings} to compute probability scores for edges:

\begin{equation} \label{eq:Fermi-dirac}
    p((i, j) \in \mathcal{E} : x_i^{L,\E}, x_j^{L,\E})=[e^{(||\vx_i^{L,\E}-\vx_j^{L,\E} ||_2^2-r)/t}+1]^{-1}, \ \text{where $L$ is the last layer}.
\end{equation}

It is important to emphasize that Eq.~\ref{eq:Fermi-dirac} uses Euclidean distance, rather than the hyperbolic distance typically employed in standard HGCN models.

\textbf{Node classification.} For node classification, we follow the HGCN setup, using Euclidean multinomial logistic regression on the output. For the PubMed dataset, we pretrain the embeddings using the link prediction task, followed by a shallow model for classification. We did not report time performance for node classification on this dataset, as the pretraining step influences the overall time, which differs from a direct classification task.

\textbf{Time performance.} We conducted a study on the time performance of the models under consideration. This analysis evaluates the computational efficiency of each model by measuring the time required per epoch during training and inference. By focusing on time per epoch, we aim to provide a consistent and comparable metric to assess the scalability and practical applicability of the models.

\section{Graph Regression} \label{ap:GraphRegression}

\subsection{Molecular Graphs}

A graph consists of nodes connected by edges, and in the context of molecules, these nodes represent atoms, while the edges correspond to the chemical bonds between them. This creates a natural graph structure for each molecule, making molecular graphs an ideal representation for use in deep learning models. The standard way to represent molecular graphs involves three primary matrices: the node feature matrix, the edge feature matrix, and the adjacency matrix, which capture the properties of atoms, bonds, and their interconnections, respectively.

In our case, we hypothesize that the molecular graphs we are working with may exhibit high hyperbolicity. This suggests that the structure of these molecules might possess properties that are better captured in hyperbolic space rather than Euclidean space. Given this, our newly proposed method, sHGCN, is particularly promising as it is designed to handle the high-dimensional, hyperbolic nature of these graphs.

\subsection{Graph Regression Pipeline}

Graph Regression is a technique where a graph is used as input to predict continuous values, typically for tasks such as molecular property prediction. In this framework, the graph structure—composed of nodes (representing atoms) and edges (representing bonds)—is processed through a neural network to output predictions based on the features of the graph.

In our case, we applied Graph Regression using our method sHGCN. This approach was particularly beneficial due to its ability to handle the complex geometric properties of molecular graphs. Unlike HGCN, where computations occur in hyperbolic space, sHGCN allows us to operate in Euclidean space. This is crucial as it enables the use of residual connections and batch normalization, which are not as straightforward to implement in the hyperbolic space without further adaptation. These techniques are essential for training deep models, as they improve gradient flow and stabilization during learning.

For the graph-level feature aggregation, we employed median global pooling, which aggregates the features of all nodes in the graph while preserving the structure of the molecule. This is followed by an MLP (Multilayer Perceptron) readout that transforms the pooled features into a prediction, making it suitable for regression tasks.

\begin{table}[ht]
\centering
\caption{Dataset Information for Graph Regression Tasks} \label{tab:ZINC_and_AQSOL}
\resizebox{\textwidth}{!}{
    \begin{tabular}{lcccccc}
    \hline
    \textbf{Dataset} & \textbf{\#Graphs} & \textbf{\#Classes} & \textbf{Avg. Nodes} & \textbf{Avg. Edges} & \textbf{Node feat. } & \textbf{Edge feat. (dim)} \\ \hline
    ZINC & 12,000 & – & 23.16 & 49.83 & Atom Type (28) & Bond Type (4) \\ 
    AQSOL & 9,823 & – & 17.57 & 35.76 & Atom Type (65) & Bond Type (5) \\ \hline
    \end{tabular}
}
\end{table}

\subsection{Experiments \& Results}

We followed the same experimental setup and used ZINC and used AQSOL datasets (see Table \ref{tab:ZINC_and_AQSOL}), as in \citep{BenchmarkingGnn} \footnote{Code available at \url{https://github.com/pol-arevalo-soler/Benchmarking_gnn}. The sHGCN method has been integrated into the original framework for benchmarking purposes.}
, ensuring a consistent comparison with traditional methods like GCNs. This allowed us to benchmark the performance of sHGCN in the context of Graph Regression and to evaluate its effectiveness in capturing the properties of molecular graphs.

\textbf{ZINC.} Table \ref{tab:ZINC_res} presents the graph regression results on the ZINC dataset. Remarkably, our model, with just 4 layers and 100k parameters, outperforms all methods with a similar parameter size, except for 3WL-GNN-E. However, our model achieves this with an average epoch time of approximately 3 seconds, significantly faster than 3WL-GNN-E.

Notably, when tested on our GPU, GCN delivered comparable results to those reported in the benchmark, though with a slightly longer epoch time. Specifically, for the 4-layer configuration, we recorded an average of 1.75 seconds per epoch, compared to the reported 1.53 seconds. This discrepancy may be attributed to differences in GPU hardware, suggesting that our method could be even faster when run on the same GPUs used in the benchmark.

As we increase the number of layers, the results do not scale as well as with other methods. This could be due to a phenomenon known as oversmoothing \citet{oversmoothing}, which affects Graph Convolutional Networks (GCNs). Oversmoothing occurs when the features of the nodes become too similar after passing through many layers, causing the model's performance to stagnate or even degrade with additional layers.

In fact, our experiments show that with 8 layers, we obtain the same results as with 16 layers, further supporting the idea of oversmoothing. This issue could be an interesting avenue for future work, where potential solutions to reduce oversmoothing in hyperbolic models could be explored.

\begin{table}[ht]
    \centering
    \caption{Results on the ZINC dataset. All results are from \citep{BenchmarkingGnn}.} \label{tab:ZINC_res}
    \resizebox{\textwidth}{!}{
        \begin{tabular}{lcccccc}
            \toprule
            \textbf{Model} & $L$ & \#\textbf{Param} & \textbf{Test MAE}$\pm$\textbf{s.d.} & \textbf{Train MAE}$\pm$\textbf{s.d.} & \#\textbf{Epoch} & \textbf{Epoch/Total} \\
            \midrule
            MLP & 4 & 108.975 & 0.706$\pm$0.006 & 0.644$\pm$0.005 & 116.75 & 1.01s/0.03hr \\
            \midrule
            \textit{vanilla GCN} & 4 & 103.077 & 0.459$\pm$0.006 & 0.343$\pm$0.011 & 196.25 & 2.89s/0.16hr \\
            & 16 & 505.079 & 0.367$\pm$0.011 & 0.128$\pm$0.019 & 197.00 & 12.78s/0.71hr \\
    
            GraphSage & 4 & 94.977 & 0.468$\pm$0.003 & 0.251$\pm$0.004 & 147.25 & 3.74s/0.15hr \\
            & 16 & 505.341 & 0.398$\pm$0.002 & 0.081$\pm$0.009 & 145.50 & 16.61s/0.68hr \\
            \midrule
            GCN & 4 & 103.077 & 0.416$\pm$0.006 & 0.313$\pm$0.011 & 159.50 & 1.53s/0.07hr \\
            & 16 & 505.079 & \textcolor{purple}{0.278$\pm$0.003} & 0.101$\pm$0.011 & 159.25 & 3.66s/0.16hr \\
            
            MoNet & 4 & 106.002 & 0.397$\pm$0.010 & 0.318$\pm$0.016 & 188.25 & 1.97s/0.10hr \\
            & 16 & 504.013 & 0.292$\pm$0.012 & 0.093$\pm$0.014 & 171.75 & 10.82s/0.52hr \\
            
            GAT & 4 & 102.385 & 0.475$\pm$0.007 & 0.317$\pm$0.006 & 137.50 & 2.93s/0.11hr \\
            & 16 & 531.345 & 0.384$\pm$0.004 & 0.067$\pm$0.004 & 144.00 & 12.98s/0.53hr \\
            
            GatedGCN & 4 & 105.735 & 0.435$\pm$0.003 & 0.287$\pm$0.014 & 173.50 & 5.76s/0.28hr \\
            GatedGCN-E & 4 & 105.875 & 0.375$\pm$0.003 & 0.236$\pm$0.007 & 194.75 & 5.37s/0.29hr \\
            & 16 & 504.309 & \textcolor{purple}{0.282$\pm$0.015} & 0.074$\pm$0.016 & 166.75 & 20.50s/0.96hr \\
            GatedGCN-E-PE & 16 & 505.011 & \textcolor{red}{0.214$\pm$0.013} & 0.067$\pm$0.019 & 185.00 & 10.70s/0.56hr \\
            \midrule
            GIN & 4 & 103.079 & 0.387$\pm$0.015 & 0.319$\pm$0.015 & 153.25 & 2.29s/0.10hr \\
            & 16 & 509.549 & 0.526$\pm$0.051 & 0.444$\pm$0.039 & 147.00 & 10.22s/0.42hr
            \\
            RingGNN & 2 & 97.978 & 0.526$\pm$0.051 & 0.444$\pm$0.039 & 147.00 & 10.22s/0.42hr \\
            RingGNN-E & 2 & 104.403 & 0.363$\pm$0.020 & 0.243$\pm$0.025 & 95.00 & 366.29s/9.76hr \\
            & 2 & 527.283 & 0.353$\pm$0.019 & 0.236$\pm$0.019 & 79.75 & 293.94s/6.63hr \\
            & 8 & 510.305 & Diverged & Diverged & Diverged & Diverged \\
            
            3WL-GNN & 3 & 102.150 & 0.407$\pm$0.028 & 0.272$\pm$0.037 & 111.25 & 286.23s/8.88hr \\
            3WL-GNN-E & 3 & 103.098 & \textcolor{purple}{0.256$\pm$0.054} & 0.140$\pm$0.044 & 117.25 & 334.69s/10.90hr \\
            & 8 & 507.603 & 0.303$\pm$0.068 & 0.173$\pm$0.041 & 120.25 & 329.49s/11.08hr \\
            & 8 & 582.824 & 0.303$\pm$0.057 & 0.246$\pm$0.043 & 52.50 & 811.27s/12.15hr \\
            \bottomrule
            \textbf{Our proposed model:} \\
            sHGCN & 4 & 103.077 & 0.360$\pm$0.014 & 0.265$\pm$0.013 & 199.00 & 3.28s/0.18hr \\
            & 16 & 505.079 & 0.302$\pm$0.005 & 0.125$\pm$0.016 & 153.00 & 11.31s/0.48hrs \\
            \bottomrule 
            
        \end{tabular}
    }
    \footnotesize{\#Param values are shown in thousands (k).}
\end{table}

\textbf{AQSOL.} Table \ref{tab:AQSOL_res} shows the results of Graph Regression for the AQSOL dataset. The results are very similar to the ZINC case, where our model outperforms the GCN, MoNet, and GAT models. However, in this case, the performance is slightly lower than all the GatedGCN models. In this dataset, the oversmoothing problem becomes more evident in all models. In fact, in the proposed model, the performance even decreases as the number of layers increases.

\begin{table}[ht]
\centering
\caption{Results on the AQSOL dataset. All results are from \citep{BenchmarkingGnn}.} \label{tab:AQSOL_res}
\resizebox{\linewidth}{!}{
    \begin{tabular}{lcccccc}
    \hline
    \textbf{Model} & \textbf{L} & \textbf{\#Param} & \textbf{Test MAE $\pm$ s.d.} & \textbf{Train MAE $\pm$ s.d.} & \textbf{Epochs} & \textbf{Epoch/Total Time} \\ \hline
    MLP & 4 & 114.525 & 1.744$\pm$0.016 & 1.413$\pm$0.042 & 85.75 & 0.61s/0.02hr \\ \hline
    vanilla GCN & 4 & 108.442 & 1.483$\pm$0.014 & 0.791$\pm$0.034 & 110.25 & 1.14s/0.04hr \\ 
     & 16 & 511.443 & 1.458$\pm$0.011 & 0.567$\pm$0.027 & 121.50 & 2.83s/0.10hr \\ 
    GraphSage & 4 & 109.620 & 1.431$\pm$0.010 & 0.666$\pm$0.027 & 106.00 & 1.51s/0.05hr \\
     & 16 & 509.078 & 1.402$\pm$0.013 & 0.402$\pm$0.013 & 110.50 & 3.20s/0.10hr \\ \hline
    GCN & 4 & 108.442 & 1.372$\pm$0.020 & 0.593$\pm$0.030 & 135.00 & 1.28s/0.05hr \\
     & 16 & 511.443 & 1.333$\pm$0.013 & 0.382$\pm$0.018 & 137.25 & 3.31s/0.13hr \\ 
    MoNet & 4 & 109.332 & 1.395$\pm$0.027 & 0.557$\pm$0.022 & 125.50 & 1.68s/0.06hr \\
     & 16 & 507.750 & 1.501$\pm$0.056 & 0.444$\pm$0.024 & 110.00 & 3.62s/0.11hr \\ 
    GAT & 4 & 108.289 & 1.441$\pm$0.023 & 0.678$\pm$0.021 & 104.50 & 1.92s/0.06hr \\
     & 16 & 540.673 & 1.403$\pm$0.008 & 0.386$\pm$0.014 & 111.75 & 4.44s/0.14hr \\ 
    GatedGCN & 4 & 108.325 & 1.352$\pm$0.034 & 0.576$\pm$0.056 & 142.75 & 2.28s/0.09hr \\
     & 16 & 507.039 & 1.355$\pm$0.016 & 0.465$\pm$0.038 & 99.25 & 5.52s/0.16hr \\ 
    GatedGCN-E & 4 & 108.535 & 1.295$\pm$0.016 & 0.544$\pm$0.033 & 116.25 & 2.29s/0.08hr \\
     & 16 & 507.273 & 1.308$\pm$0.013 & 0.367$\pm$0.012 & 110.25 & 5.61s/0.18hr \\ 
    GatedGCN-E-PE & 16 & 507.663 & \textcolor{red}{0.996$\pm$0.008} & 0.372$\pm$0.016 & 105.25 & 5.70s/0.30hr \\ \hline
    GIN & 4 & 107.149 & 1.894$\pm$0.024 & 0.660$\pm$0.027 & 115.75 & 1.55s/0.05hr \\
     & 16 & 514.137 & 1.962$\pm$0.058 & 0.850$\pm$0.054 & 128.50 & 3.97s/0.14hr \\ 
    RingGNN & 2 & 116.643 & 20.264$\pm$7.549 & 0.625$\pm$0.018 & 54.25 & 113.99s/1.76hr \\ 
    RingGNN-E & 2 & 123.157 & 3.769$\pm$1.012 & 0.470$\pm$0.022 & 63.75 & 125.17s/2.26hr \\ 
     & 2 & 523.935 & Diverged & Diverged & Diverged & Diverged \\ 
     & 8 & - & Diverged & Diverged & Diverged & Diverged \\ 
    3WLGNN & 3 & 110.919 & 1.154$\pm$0.050 & 0.434$\pm$0.026 & 66.75 & 130.92s/2.48hr \\
     & 3 & 525.423 & 1.108$\pm$0.036 & 0.405$\pm$0.031 & 70.75 & 131.12s/2.62hr \\ 
    3WLGNN-E & 3 & 112.104 & \textcolor{purple}{1.042$\pm$0.064} & 0.307$\pm$0.024 & 68.50 & 139.04s/2.70hr \\
     & 3 & 528.123 & \textcolor{purple}{1.052$\pm$0.034} & 0.287$\pm$0.023 & 67.00 & 140.43s/2.67hr \\ 
     & 8 & - & Diverged & Diverged & Diverged & Diverged \\
    \hline
    \textbf{Our proposed model:} \\
    sHGCN & 4 & 108.442 & 1.357$\pm$0.008 & 0.526$\pm$0.027 & 142.00 & 2.59s/0.10hr \\
    & 16 & 511.443 & 1.400$\pm$0.022 & 0.420$\pm$0.012 & 112.50 & 9.90s/0.31hr
 \\ \hline
    \end{tabular}
}
\footnotesize{\#Param values are shown in thousands (k).}
\end{table}

\end{document}